\documentclass[a4paper]{article}

\usepackage[english]{babel}
\usepackage[utf8]{inputenc}
\usepackage{authblk}
\usepackage{amsmath}
\usepackage{amssymb}
\usepackage{graphicx}
\usepackage{hyperref}
\usepackage{geometry}
\usepackage{lscape}
\usepackage{color}
\usepackage[colorinlistoftodos]{todonotes}
\usepackage[normalem]{ulem}
\usepackage[lined, boxed,linesnumbered,ruled,vlined]{algorithm2e}
\usepackage{subcaption}
\usepackage{rotating}

\usepackage[square,numbers]{natbib} % for citation
\usepackage{pifont} % checked boxes
\usepackage{enumitem}

\usepackage{amsfonts}
\usepackage{amsthm} % For theorems
\newtheorem{theorem}{Theorem}
\newtheorem{lemma}{Lemma}

\definecolor{purpleNew}{rgb}{0.710, 0.631, 0.722}
\definecolor{yellowNew}{rgb}{0.996, 0.945, 0.706}
\definecolor{tealNew}{rgb}{0.655, 0.827, 0.820}

\providecommand{\keywords}[1]
{
  \small	
  \textbf{\textit{Keywords---}} #1
}

\title{FoLDTree: A ULDA-Based Decision Tree Framework for Efficient Oblique Splits and Feature Selection}
\author[1]{Siyu Wang}
\author[1]{Kehui Yao}
\affil[1]{Department of Statistics\\ University of Wisconsin - Madison}
\date{} % \date{\today}

\begin{document}

\maketitle
\setlength{\parindent}{0pt} % no indent

%------------------------------------------------
%------------------------------------------------
\section*{Abstract}

Traditional decision trees are limited by axis-orthogonal splits, which can perform poorly when true decision boundaries are oblique. While oblique decision tree methods address this limitation, they often face high computational costs, difficulties with multi-class classification, and a lack of effective feature selection. In this paper, we introduce LDATree and FoLDTree, two novel frameworks that integrate Uncorrelated Linear Discriminant Analysis (ULDA) and Forward ULDA into a decision tree structure. These methods enable efficient oblique splits, handle missing values, support feature selection, and provide both class labels and probabilities as model outputs. Through evaluations on simulated and real-world datasets, LDATree and FoLDTree consistently outperform axis-orthogonal and other oblique decision tree methods, achieving accuracy levels comparable to the random forest. The results highlight the potential of these frameworks as robust alternatives to traditional single-tree methods.

\keywords{Oblique decision trees, Linear Discriminant Analysis (LDA), Machine Learning, Feature Selection, Classification}

%------------------------------------------------
%------------------------------------------------
\section{Introduction}
\label{sec:Introduction}

Linear discriminant analysis (LDA), which assumes Gaussian densities on the covariates, is a powerful linear method for classification problems. LDA aims to find linear combinations of features that can best separate the groups by maximizing the between-group and within-group variance ratio. Another popular classifier with a linear decision boundary is the decision tree, which recursively partitions the sample space into rectangular regions.\\

One of the biggest problems of the tree-based method is that it usually cuts in an axis-orthogonal direction. For numerical covariates, the (univariate) splitting rule for each non-terminal node has the form $x_i \le a$. For categorical covariates, the splitting rule can be written as $x_i \in S$, where $S \subsetneq \{ c_{1}, c_{2}, \dots, c_{k}\}$ and $\{ c_{1}, c_{2}, \dots, c_{k}\}$ are the full set of levels of the variable $X_i$. The decision tree loses its effectiveness when the real decision boundary is not orthogonal to the axes. e.g., Figure~\ref{Figure:DToblique}(a). The decision tree has 10 splits and uses a staircase function to approximate the boundary line. Additionally, this approximation's accuracy decreases as the dimensions increase, e.g., Figure~\ref{Figure:DToblique}(b). The gap between the true boundary and the fitted boundary from the decision tree becomes larger. LDA, which can directly fit a hyperplane decision boundary in the high-dimension space, is more suitable for this problem.\\

\begin{figure}[htbp]
    \centering
    \begin{subfigure}[b]{0.45\textwidth}
        \centering
        \includegraphics[width=\textwidth]{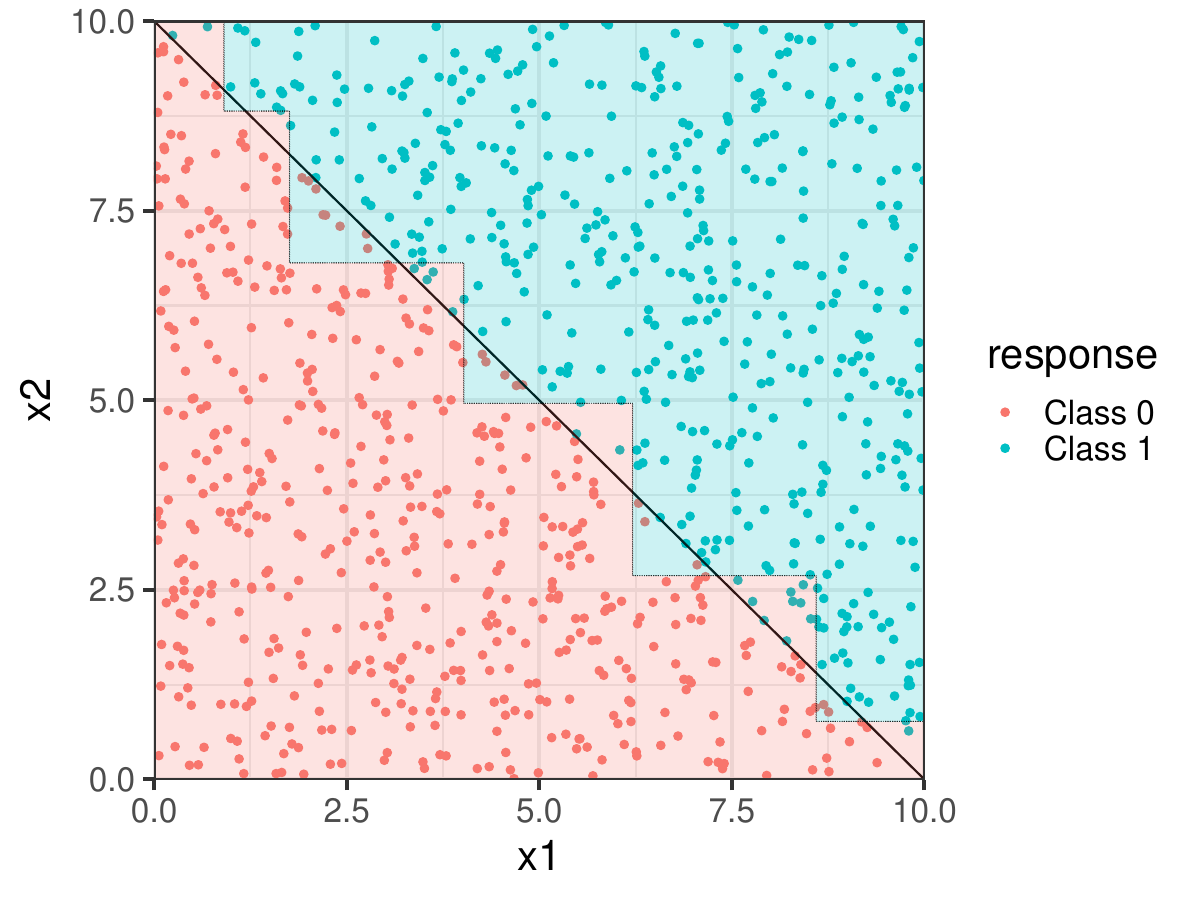}
        \caption{2D decision boundary}
    \end{subfigure}%
    \hfill
    \begin{subfigure}[b]{0.5\textwidth}
        \centering
        \includegraphics[width=\textwidth]{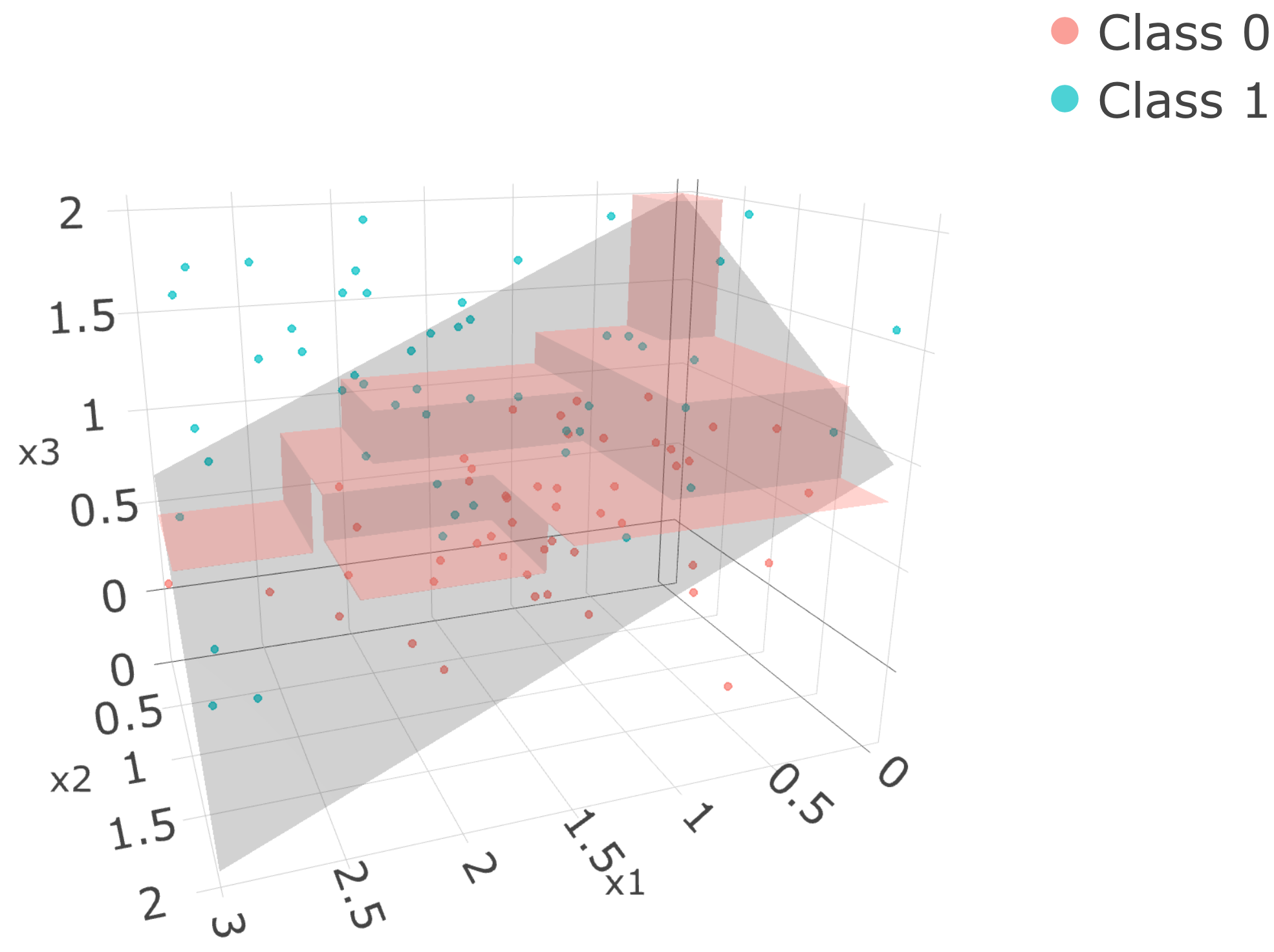}
        \caption{3D decision boundary}
    \end{subfigure}
    \caption{Decision boundaries from the decision tree (\texttt{rpart} in R) when the decision boundary is not orthogonal to the axes (Section~\ref{sec:Introduction}). (a) 2D decision boundary with the actual boundary represented by a straight line, and the staircase boundary fitted by \texttt{rpart}. The colored background represents the predicted regions. (b) 3D decision boundary, where the gray surface represents the actual boundary, and the red step-function surface is fitted by \texttt{rpart}}
    \label{Figure:DToblique}
\end{figure}

On the other hand, LDA uses the Gaussian density assumption, and the decision boundary for each class is usually a polygon. Therefore, LDA's performance may decline when the class centroids are closely situated together or when symmetry is present, e.g., Figure~\ref{Figure:LDAsymmetry}(a). By first applying tree-based methods to split the data on the x variable, the class centroids can be separated and the symmetry destroyed. This enables LDA to carry out its classification task more effectively (Figure~\ref{Figure:LDAsymmetry}(b)).\\

\begin{figure}[htbp]
    \centering
    \begin{subfigure}[b]{0.5\textwidth}
        \centering
        \includegraphics[width=\textwidth]{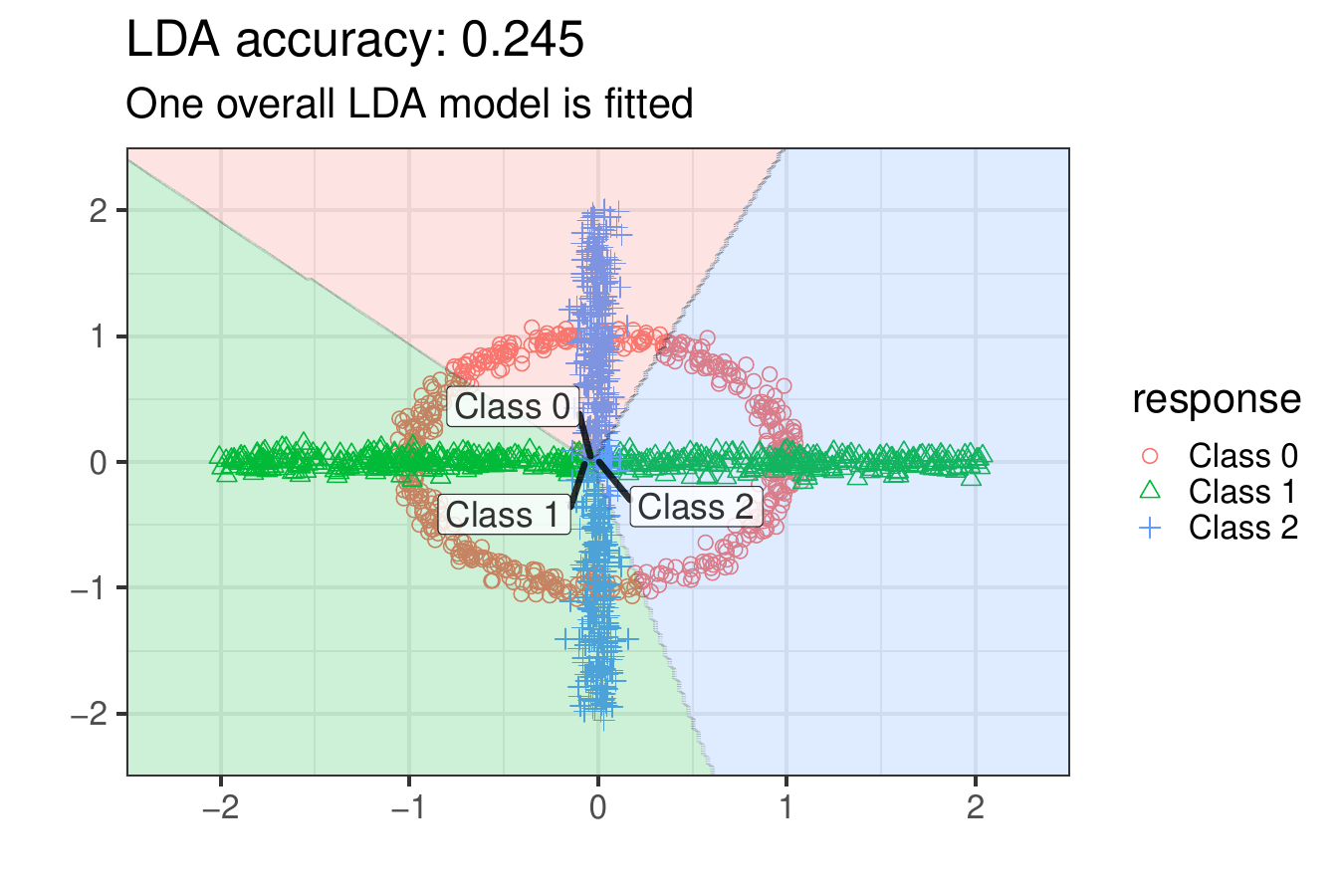}
        \caption{LDA before splitting}
    \end{subfigure}%
    \hfill
    \begin{subfigure}[b]{0.5\textwidth}
        \centering
        \includegraphics[width=\textwidth]{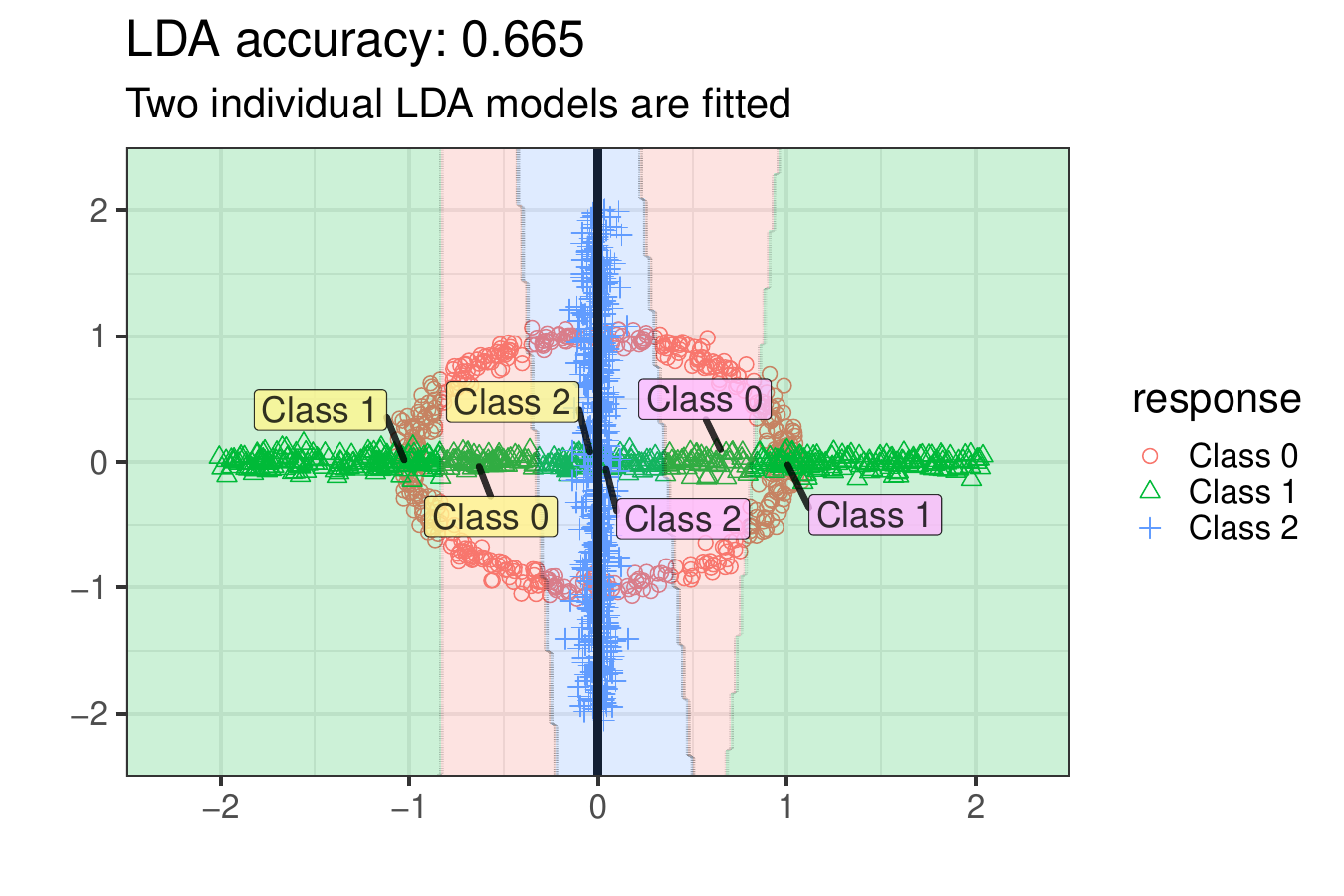}
        \caption{LDA after splitting}
    \end{subfigure}
    \caption{Decision boundaries from LDA. (Section~\ref{sec:Introduction}). The colored background represents LDA prediction regions, and the labels indicate the class centroids. (a) LDA decision boundary before splitting. (b) LDA decision boundary after splitting on $x \leq 0$}
    \label{Figure:LDAsymmetry}
\end{figure}

Another problem is missing values. LDA's decision boundary is based on linear combinations of covariates, and missing values in any covariate will make the boundary uncomputable. In this paper, we propose decision tree frameworks based on recursive LDA that can handle missing values and high-dimensional data. This paper is organized as follows: Section \ref{sec:relatedWork} provides a literature review of decision trees with oblique splits, highlighting associated challenges. Section \ref{sec:algorithm} introduces our proposed framework, detailing the splitting rules, node models, stopping rules, and approach to handling missing values. Section \ref{sec:Empirical} presents analyses based on synthetic and real datasets, along with use cases of our proposed methods. We conclude in Section \ref{sec:Conclusion}.

%------------------------------------------------
%------------------------------------------------
\section{Related Work}
\label{sec:relatedWork}

The way we combine LDA and the decision tree is by keeping the decision tree structure but using LDA to determine the splits. Let $\mathbf{X} =\left(\mathbf{X}_1, \mathbf{X}_2, \ldots, \mathbf{X}_p\right)$ represent a feature vector in a $p$-dimensional space, where each $X_i$ is a numerical feature. An oblique split based on a linear combination of multiple features is represented by:

\begin{equation}
	w_1 X_1+w_2 X_2+\cdots+w_p X_p \le t
\end{equation}

where:
\begin{itemize}
\item $\mathbf{w}=\left(w_1, w_2, \ldots, w_p\right)$ are the coefficients that define the orientation of the hyperplane.
\item $t$ is a threshold that defines the position of the hyperplane.
\end{itemize}

Several methods in the literature also use oblique splits within the decision tree structure. Broadly, these methods can be categorized into the three groups below.

%------------------------------------------------
\subsection{Greedy Search Over All Possible Linear Combinations}
\label{subsec:method1Previous}

This family of methods first initializes the parameters $\mathbf{w}$ and $t$ randomly, then uses optimization techniques to update the parameters until the predefined loss function is minimized. These methods typically differ in their optimization approaches.\\

CART-LC \cite{breiman1984classification} uses a deterministic hill-climbing algorithm, which may lead to a local minimum. DT-SE \cite{john1995robust}, which uses soft entropy as the loss function, faces a similar issue. To mitigate this, SADT \cite{heath1993induction} uses simulated annealing, but the added randomization significantly reduces computational efficiency. OC1 \cite{murthy1994system} combines deterministic hill-climbing with randomization, starting with a deterministic search to reach a local minimum, followed by randomization to escape it.\\
 
In practice, the computational burden of this approach is so high that it typically limits the search to a restricted parameter space, leading to compromised performance.

%------------------------------------------------
\subsection{Search for Axis-Orthogonal Split in the Transformed Space}

The computational burden of the first family of methods comes from its unrestricted parameter space, which leads us to the second family, where the parameter space is more restricted. These methods first generate effective linear combinations, and then find the best univariate cut among these combinations. \\

LDCT \cite{todeschini1992linear}, Linear Tree \cite{gama1999linear}, and FDT \cite{lopez2013fisher} use LDA to transform the data into linear discriminant scores and then find the best cut among these scores. LDTS \cite{li2003multivariate} also uses LDA, but instead of fitting one LDA model with all variables, it fits an LDA for each possible subset of variables. Given $M$ available variables, LDTS fits $2^M - 1$ LDA models and uses Tabu search to find the best LD score to cut. PPtree \cite{lee2013pptree} offers three data transformation methods, including LDA, $L_r$ \cite{lee2005projection}, and PDA \cite{lee2010projection}. HHCART \cite{wickramarachchi2016hhcart}, on the other hand, uses Householder reflection based on the orientation of the dominant eigenvector of the covariance matrix, while HHCART(G) \cite{wickramarachchi2019reflected} applies Householder reflection to a modified Geometric Decision Tree (GDT; \cite{manwani2011geometric}) angle bisector.\\

This family relies on the effectiveness of the transformation method. Most methods in this group still require a certain amount of search and some methods cannot naturally handle multi-class classification.

%------------------------------------------------
\subsection{Directly Adopting the Decision Boundary from an External Classifier}

This family of methods is straightforward; it first fits an external linear classifier and then directly uses its decision boundary as the splitting criterion.\\

FACT \cite{loh1988tree} uses the combination of Principal Component Analysis (PCA) and LDA as the classifier. QUEST \cite{loh1997split} uses Quadratic Discriminant Analysis (QDA), GUIDE \cite{loh2009improving} has the option to use the best bivariate LDA and Stree \cite{montanana2021stree} uses the Support Vector Machine (SVM).\\

This family relies on the effectiveness of the external classifier even more. However, no search is required.

%------------------------------------------------
%------------------------------------------------
\section{The Proposed Algorithm}
\label{sec:algorithm}

Our proposed method belongs to the third family and makes full use of LDA's effectiveness. It handles missing values, includes a built-in variable selection framework, and can output both predicted classes and posterior probabilities.

%------------------------------------------------
\subsection{From LDA to ULDA and Forward ULDA}

In traditional LDA, the goal is to find linear projections that best separate classes. These decision boundaries are found by maximizing the ratio of between-class scatter matrix to within-class scatter matrix, similar to optimizing a signal-to-noise ratio. Fisher's criterion aims to find transformation vectors $\mathbf{w}_{M \times 1}$ that maximizes the ratio:

\begin{equation}
\arg \max _{\mathbf{w}} \frac{\mathbf{w}^T \mathbf{S}_B \mathbf{w}}{\mathbf{w}^T \mathbf{S}_W \mathbf{w}}
%\label{eq:optimizationProblem}
\end{equation}

where $\mathbf{S}_B$ is the between-class scatter matrix, and $\mathbf{S}_W$ is the within-class scatter matrix. The optimal $\mathbf{W}$ is derived by solving an eigenvalue problem involving $\mathbf{S}_W^{-1}\mathbf{S}_B$. The resulting $\mathbf{W}$ projects the original data $\mathbf{X}$ into orthogonal linear discriminant scores $\mathbf{X}\mathbf{w}_i$, ranked in descending order of their signal-to-noise ratios.\\

However, challenges arise when the within-class scatter matrix, $\mathbf{S}_W$, is not invertible, such as when there are more variables than observations or when variables are linearly dependent. This issue has been widely addressed in the literature with several solutions proposed \cite{mai2013review, tharwat2017linear}. Among these solutions, ULDA \cite{ye2005characterization, howland2003structure} is a promising candidate for integration into our decision tree framework. ULDA uses the following criterion:

\begin{equation}
		\mathbf{W}=\arg \max _{\mathbf{W}^T \mathbf{S}_T \mathbf{W}=I} \operatorname{trace}\left(\left(\mathbf{W}^T \mathbf{S}_T \mathbf{W}\right)^{+}\left(\mathbf{W}^T \mathbf{S}_B \mathbf{W}\right)\right),
		\label{eq:ULDA}
\end{equation}

where $\left(\mathbf{W}^T \mathbf{S}_T \mathbf{W}\right)^{+}$ denotes the Moore-Penrose inverse of $\mathbf{W}^T \mathbf{S}_T \mathbf{W}$, and $\mathbf{S}_T$ is the total scatter matrix. This ensures that the ULDA solution always exists, and \cite{ye2004feature} shows that ULDA is equivalent to classical LDA when $\mathbf{S}_T$ is nonsingular. Unlike Regularized LDA \cite{friedman1989regularized} or Penalized LDA \cite{witten2011penalized}, ULDA does not require hyperparameter tuning. This enables us to fit multiple LDA models within the tree structure without concerns about computational efficiency.\\

By default, LDA (or ULDA) uses all variables to find the transformation matrix $\mathbf{W}$, making it prone to overfitting. Most decision trees using LDA models overlook this potential pitfall. LDTS mitigates this by testing all possible subsets, resulting in another significant computational burden. We need a more efficient way to rank the variables and select the appropriate number of variables for inclusion in LDA, making stepwise LDA a promising solution. By using stepwise LDA, we eliminate noise variables, enhance interpretability, and achieve better computational efficiency with a smaller set of variables. However, traditional stepwise LDA relies on Wilks' $\Lambda$ and is less effective when $\mathbf{S}_W$ becomes singular, as mentioned in \cite{wang2024newforwarddiscriminantanalysis}. Therefore, we adopt the same approach and use forward ULDA as our solution to the overfitting problem in ULDA. In the remainder of this paper, LDA refers to either ULDA or forward ULDA.

%------------------------------------------------
\subsection{Splitting Rule}
\label{subsec:splitting}

We propose two methods in this paper, LDATree and FoLDTree. LDATree uses ULDA as the splitting rule and FoLDTree uses forward ULDA as the splitting rule. Within each node, we first fit an LDA model using the data in that node. Then, we create $ J^{\prime} $ subnodes, where $ J^{\prime} $ represents the number of distinct predicted classes. If $J^{\prime} = 1$, the splitting is stopped. Otherwise, each observation is distributed to a branch along with others in the same predicted class. Occasionally, we encounter $ J^{\prime} < J $, indicating that some classes are not predicted (hidden) in this training set, yet these classes might appear in predictions of the testing set. If that happens, we assign the observation to the class with the highest posterior probability among the non-hidden classes.\\

Some methods from Section~\ref{sec:relatedWork} cannot naturally handle scenarios where the response variable has more than two classes. They first combine the classes into two superclasses and then solve it as a binary classification problem. In contrast, LDA can naturally handle multi-class problems, and we inherit this advantage in our framework, allowing it to produce a decision tree with multi-way splits in such cases.\\

One advantage of LDA is that its decision boundary is derived from the Gaussian assumption with equal covariance across different classes. Since it relies on the distribution rather than individual data points, it is less prone to overfitting and more computationally efficient than methods like SVM, which depends on points near the decision boundary. However, this robustness can sometimes lead to blind spots. Next, we will illustrate how we make LDA splitting more effective and prevent premature stopping when a dominant class is present.\\

Figure~\ref{Figure:splitDom}(a) illustrates a simulated scenario where class $A$ is the dominant class (80.9\%). If we use the estimated prior by default, LDA predicts all observations as class $A$, and the posterior probabilities are shown in Figure~\ref{Figure:splitDom}(b), where the density of class $A$ dominates that of class $B$ everywhere. While predicting everyone as class $A$ makes LDA a good classifier, it does not make it an effective splitting tool. A good splitting tool in this scenario will continue splitting until class $A$ and class $B$ are separated. Initially, the density plots of class $A$ and class $B$ have the same spread due to the constant variance assumption in LDA. However, the unequal prior lowers the density of class $B$. To address this, we assign equal priors to both classes.\\

\begin{figure}[htbp]
    \centering
    \begin{subfigure}[b]{0.5\textwidth}
        \centering
        \includegraphics[width=\textwidth]{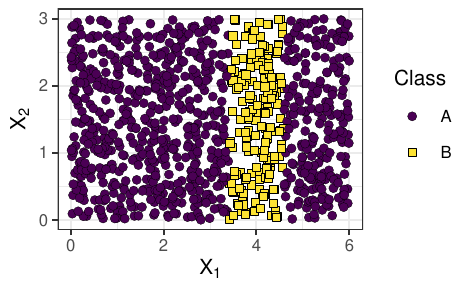}
        \caption{Scatter plot}
    \end{subfigure}%
    \hfill
    \begin{subfigure}[b]{0.5\textwidth}
        \centering
        \includegraphics[width=\textwidth]{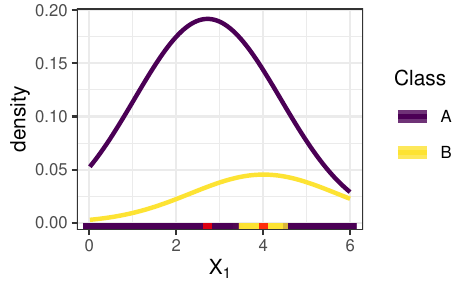}
        \caption{Density plot based on LDA's assumption}
    \end{subfigure}
        \begin{subfigure}[b]{0.5\textwidth}
        \centering
        \includegraphics[width=\textwidth]{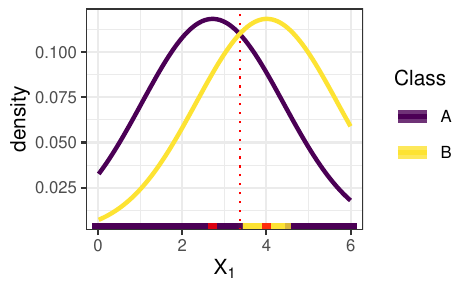}
        \caption{Density plot based on LDA's assumption with equal priors}
    \end{subfigure}%
    \hfill
    \begin{subfigure}[b]{0.5\textwidth}
        \centering
        \includegraphics[width=\textwidth]{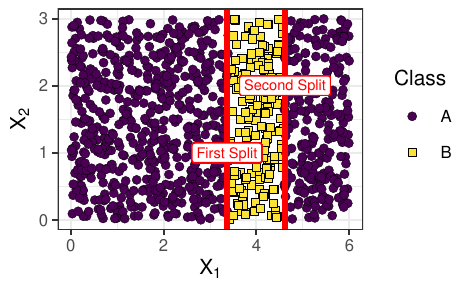}
        \caption{First two splits in LDATree}
    \end{subfigure}
	\caption{Simulated scenario where LDA splitting fails and our approach to address it (Section \ref{subsec:splitting}). (a) Scatter plot showing class $A$ dominating class $B$. (b) Posterior probabilities from LDA, where class $A$'s dominance leads to no splits under the estimated priors. Red dots on the x-axis represent class centers. (c) Posterior probabilities from LDA with equal priors, allowing LDA to split at the intersection of the two densities. (d) First and second splits in LDATree, demonstrating effective class separation}
    \label{Figure:splitDom}
\end{figure}

Figure~\ref{Figure:splitDom}(c) shows the new density plot under equal priors, where LDA can now split at the intersection of the two densities (red dotted line). Figure~\ref{Figure:splitDom}(d) illustrates the first and second splits when fitting the LDATree, demonstrating decent performance. In our implementation, we switch to equal priors when the Gini Index of the predicted classes in a particular node falls within the range $(0, 0.1]$.

%------------------------------------------------
\subsection{Node Model}
\label{subsec:nodeModel}

The node model reflects how predictions are made in a terminal node. The simplest node model is the plurality rule, where all observations in the node are predicted as the plurality class. We use ULDA as the node model in LDATree and forward ULDA in FoLDTree. Both models fall back to the plurality rule if the LDA model does not achieve better training accuracy in the current node.\\

There are several reasons to use LDA as the node model:
\begin{enumerate}
	\item It has no additional time cost, as the LDA model is already fitted during the split-finding process. If the LDA split is ineffective or pruned, it serves as the node model.
	\item It provides posterior probabilities rather than just the predicted class. Using LDA as the node model at depth $d$ is similar to splitting the node with LDA at depth $d$ and applying the plurality rule in its child nodes at depth $(d+1)$. However, the plurality rule only outputs the predicted class, whereas LDA, which assumes Gaussian densities, is a likelihood-based classifier. The ability to output posterior probabilities allows for further inferences and enables comparison with results from other likelihood-based classifiers.
	\item It acts as a random split when no effective LDA cut is found. Some splitting rules introduce randomness to reduce greediness and improve overall performance, as seen in \cite{geurts2006extremely}. LDA splits seamlessly combine both greedy and random splitting rules.
\end{enumerate}

%------------------------------------------------
\subsection{Stopping Rule}
\label{subsec:stoppingRule}

The stopping rule can be viewed as the model selection tool for the decision tree. If the process of splitting stops too late, then the tree is likely to overfit. If it stops too soon, the tree has not captured the information in the dataset yet. Therefore, it's crucial to establish a rule that stops the splitting process once the splits are ineffective. Here, we proposed a new measure for the strength of splits.\\

Theoretically, the best way to validate a split's performance is by using external datasets like the validation or testing set. However, by splitting the original dataset, the tree will be built on a smaller dataset and might not be as good as if we use the full dataset. Therefore, most well-established tree methods measure the strength of a split by training (or resubstitution) error. The question to be asked here is whether replacing the subtree $T_t$ (originating from node $t$) with a terminal node $t$ significantly decreases the training error. For example, CART \cite{breiman1984classification} uses $\alpha$ to measure the strength of a split, which is defined as

\begin{equation}
\alpha = \frac{R(t)-R\left(T_t\right)}{\left|\tilde{T}_t\right|-1},
\label{eq:CARTalpha}
\end{equation}

where $R(T_t)$ is the training error of tree $T_t$ and $\left|\tilde{T}_t\right|$ is the number of the terminal nodes in the subtree $T_t$. From the formula, we know that it measures how each additional terminal node contributes to the reduction in overall training error. However, this measure focuses solely on the absolute amount of error reduction, rather than the error reduction relative to the sample size. Here is an example to illustrate the difference. Suppose we want to measure the strength of the split in Figure~\ref{Figure:stoppingNew}. Figure~\ref{Figure:stoppingNew}(a) shows the data pattern and the hypothetical first split in the decision tree, while Figure~\ref{Figure:stoppingNew}(b) illustrates the tree structure. There are 100 data points from each class. Suppose the decision rule at the root node is the plurality rule; it will predict that everyone belongs to the same class (either $A$ or $B$), resulting in 100 training errors. In the left node ($X_1 \leq 0$), there are 50 $A$ and 50 $B$ randomly distributed, and the training error in this node using the plurality rule is 50. In the right node ($X_1 > 0$), suppose we use an LDA model. In this case, classes A and B can be perfectly separated, resulting in no training error. Based on equation \eqref{eq:CARTalpha}, the strength of this split is $\alpha = \frac{100 - 50}{2 - 1} = 50$.\\

\begin{figure}[htbp]
    \centering
    \begin{subfigure}[b]{0.5\textwidth}
        \centering
        \includegraphics[width=\textwidth]{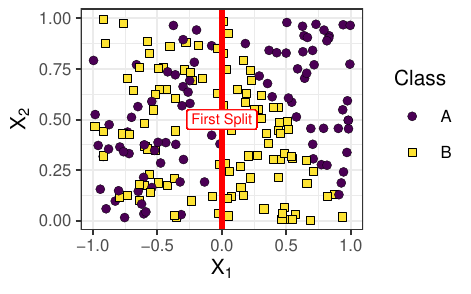}
        \caption{Scatter plot with 200 data points}
    \end{subfigure}%
    \hfill
    \begin{subfigure}[b]{0.5\textwidth}
        \centering
        \includegraphics[width=\textwidth]{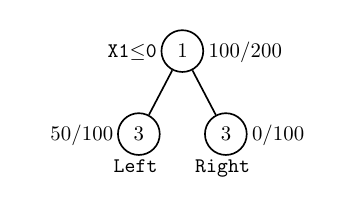}
        \caption{Decision tree corresponding to (a)}
    \end{subfigure}
        \begin{subfigure}[b]{0.5\textwidth}
        \centering
        \includegraphics[width=\textwidth]{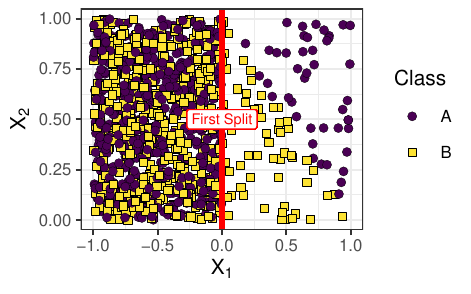}
        \caption{Scatter plot with 1200 data points}
    \end{subfigure}%
    \hfill
    \begin{subfigure}[b]{0.5\textwidth}
        \centering
        \includegraphics[width=\textwidth]{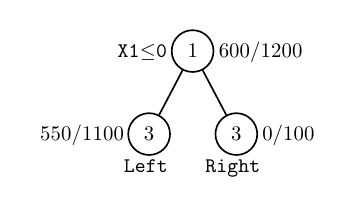}
        \caption{Decision tree corresponding to (c)}
    \end{subfigure}
	\caption{Simulated scenario where the split strength $\alpha$ from CART may be misleading (Section \ref{subsec:stoppingRule}). In the tree plots, the numbers $N_{\text{mis}}/N_{\text{total}}$ next to the nodes indicate that out of $N_{\text{total}}$ samples, $N_{\text{mis}}$ were misclassified (training errors)}
    \label{Figure:stoppingNew}
\end{figure}

Now, suppose we have a slightly different scenario. In addition to the 200 data points, we have 1000 more data points that satisfy $X_1 \leq 0$, and the results are summarized in Figure~\ref{Figure:stoppingNew}(c)(d). Here, the strength of the split is $\alpha = \frac{600 - 550}{2 - 1} = 50$, which is the same as in the scenario shown in Figure~\ref{Figure:stoppingNew}(a). However, these two splits have very different predictive power. The split in Figure~\ref{Figure:stoppingNew}(a) increases the training accuracy from $0.5$ to $150/200 = 0.75$. In contrast, the split in Figure~\ref{Figure:stoppingNew}(c), despite having the same $\alpha$ as in Figure~\ref{Figure:stoppingNew}(a), increases the training accuracy only from $0.5$ to $650/1200 = 0.54$. As you can imagine, if we keep adding points to the region $X_1 \leq 0$, the effectiveness of the split gradually decreases, but $\alpha$ remains unchanged, motivating us to find an alternative measure.\\

Suppose there are $N_{\text{total}}$ observations in the current node. Without splitting, the training error is $N_{\text{before}}$. After the split, the training error is $N_{\text{after}}$. Assuming the tree classifier has the same predictive power on all observations, the predictions can be viewed as binomial outcomes. Measures based on the binomial distribution have been proposed in previous literature. However, they have certain drawbacks. Let $p_1$ represent the prediction accuracy before the split and $p_2$ the accuracy after the split. Pessimistic Error Pruning \cite{quinlan1987simplifying} compares $p_2$ with the proportion of the plurality class instead of $p_1$. Consequently, it is overly optimistic, sensitive to randomness, and does not integrate well with non-trivial node models like LDA. C4.5 \cite{quinlan1993c4}, on the other hand, uses the confidence bound to form the stopping rule. It prunes the current split if the lower confidence bound of $p_1$ is larger than a weighted version of the lower confidence bound of $p_2$. This comparison of confidence bounds is difficult to interpret. More importantly, while generating the confidence bound requires a predefined type I error rate, the output is a binary decision rather than a continuous measure, leaving little room for further tuning.\\

Our proposed measure is based on a two-sample z-test. We test the hypothesis $ H_0: p_1 = p_2 $ against the alternative $H_1: p_1 < p_2$. Our test statistic is:

\begin{equation}
z = \frac{N_{\text{before}} - N_{\text{after}}}{\sqrt{\frac{1}{N_{\text{total}}}(N_{\text{before}}(N_{\text{total}} - N_{\text{before}}) + N_{\text{after}}(N_{\text{total}} - N_{\text{after}}))}} \sim N(0,1).
\end{equation}

Then the $p$-value from this one-sided z-test serves as a good measure of the split strength. For the split in Figure~\ref{Figure:stoppingNew}(a), the test statistic is $5.35$ and the corresponding $p$-value equals $4.5\times10^{-8}$. For the split in Figure~\ref{Figure:stoppingNew}(c), the test statistic is $2.04$ and the corresponding $p$-value equals $0.02$. This result shows that our proposed measure can better reflect the training accuracy in this scenario.\\

Based on this proposed measure, we implement both pre-stopping and post-pruning. For the pre-stopping rule, we stop splitting when the $p$-value falls below a predefined threshold (set to $0.01$ by default). For post-pruning, we adopt Cost-Complexity Pruning from CART. If the training accuracy remains the same after one split, we would expect a $p$-value of $0.5$. Thus, we set the threshold at $0.6$ to allow the current split to introduce slightly more errors while leaving room for potentially effective splits afterward. After the tree grows to its full depth, we use cross-validation to prune the tree.

%------------------------------------------------
\subsection{Missing Values}
\label{subsec:NA}

Although the decision tree itself has several ways to bypass the missing value problem, imputation is still necessary when building LDA models. Depending on the tree size, more than 1000 LDA models may need to be fitted. Therefore, it is crucial to find a method that is both effective and efficient. Here, we use a simple imputation method that imputes the sample median for numerical variables, adds missing value indicators, and imputes categorical variables with a new level. The added missing flags are essential when the missingness is informative. When the missing pattern is random, these missing flags will not cause issues and will be deleted during the forward selection step.\\

Another factor to consider is when to perform the imputation. We can either impute once at the root node (root-node imputation) or impute at every single node (node-wise imputation). Since we impute categorical variables with a new level, the imputation location does not matter for them. The only difference arises for numerical variables, and we present the following results.

%-------------------------
% Lemma 1
\begin{lemma}[]
\label{lem:numColSpace1}
For a numerical predictor $\mathbf{X}$ containing missing values, we impute with a constant $C$ and add the missing value indicator $X^{-} = I(X = \text{NA})$. Then, the column spaces of $\{\mathbf{X}, \mathbf{X^{-}}\}$ does not depend on the choice of $C$.
\end{lemma}

\begin{proof}
We want to show that the column spaces of $\{\mathbf{X}_{c_1}, \mathbf{X^{-}}_{c_1}\}$ and $\{\mathbf{X}_{c_2}, \mathbf{X^{-}}_{c_2}\}$ are the same for any $c_1, c_2$.

Any element $\mathbf{E}$ in the column space of $\{\mathbf{X}_{c_1}, \mathbf{X^{-}}_{c_1}\}$ can be expressed as $\mathbf{E} = a\cdot\mathbf{X}_{c_1} + b\cdot\mathbf{X^{-}}_{c_1}$. For the $i$-th entry that is not missing in $\mathbf{X}$, $E_i = a \cdot X_i + b \cdot 0$. For the $j$-th entry that is missing in $\mathbf{X}$, $E_j = a \cdot c_1 + b \cdot 1 = a \cdot c_2 + (b + a \cdot c_1 - a \cdot c_2) \cdot 1$. Thus, $\mathbf{E}$ can be rewritten as $\mathbf{E} = a \cdot \mathbf{X}_{c_1} + b \cdot \mathbf{X^{-}}_{c_1} = a^* \cdot \mathbf{X}_{c_2} + b^* \cdot \mathbf{X^{-}}_{c_2}$, where $a^* = a$ and $b^* = b + a \cdot c_1 - a \cdot c_2$. Therefore, $\mathbf{E}$ belongs to the column space of $\{\mathbf{X}_{c_2}, \mathbf{X^{-}}_{c_2}\}$, showing that $\text{span}(\mathbf{X}_{c_1}, \mathbf{X^{-}}_{c_1}) \subseteq \text{span}(\mathbf{X}_{c_2}, \mathbf{X^{-}}_{c_2})$. Similarly, we have $\text{span}(\mathbf{X}_{c_2}, \mathbf{X^{-}}_{c_2}) \subseteq \text{span}(\mathbf{X}_{c_1}, \mathbf{X^{-}}_{c_1})$. Eventually, this shows that $\text{span}(\mathbf{X}_{c_1}, \mathbf{X^{-}}_{c_1}) = \text{span}(\mathbf{X}_{c_2}, \mathbf{X^{-}}_{c_2})$.
\end{proof}

%-------------------------
% Lemma 2
\begin{lemma}[]
\label{lem:numColSpace2}
For $K$ numerical predictors $\mathbf{X}_1, \mathbf{X}_2, \cdots, \mathbf{X}_K$ containing missing values, the combined column spaces of $\{\mathbf{X}_1, \mathbf{X}^{-}_1, \mathbf{X}_2, \mathbf{X}^{-}_2, \cdots, \mathbf{X}_K, \mathbf{X}^{-}_K\}$ does not depend on the choice of $C_1,C_2,\cdots,C_K$.
\end{lemma}

\begin{proof}
This result is a simple extension of Lemma \ref{lem:numColSpace1}.
\end{proof}

%-------------------------
% Lemma 3
\begin{lemma}[]
\label{lem:ULDA}
The decision rules from the ULDA model do not depend on the choice of $C_1,C_2,\cdots,C_K$.
\end{lemma}

\begin{proof}
Suppose we have two design matrices $\mathbf{X}_1$ and $\mathbf{X}_2$ where the only difference is their choices of $C_i$. Based on Lemma \ref{lem:numColSpace2} we know that they share the same column space, which means $\mathbf{X}_2 = \mathbf{X}_1 \mathbf{R}$ where $\mathbf{R}$ is full rank square matrix. Suppose the transformation matrices we find for $\mathbf{X}_1$ and $\mathbf{X}_2$ based on equation \eqref{eq:ULDA} are $\mathbf{W}_1$ and $\mathbf{W}_2$, respectively. Define

\begin{equation}
	F_{\mathbf{X}}(\mathbf{W})=\operatorname{trace}\left(\left(\mathbf{W}^T \mathbf{S}_T \mathbf{W}\right)^{+}\left(\mathbf{W}^T \mathbf{S}_B \mathbf{W}\right)\right).
\end{equation}

Then we have

\begin{equation}
	F_{\mathbf{X}_2}(\mathbf{W}_2) = F_{\mathbf{X}_1}(\mathbf{R}\mathbf{W}_2) \le F_{\mathbf{X}_1}(\mathbf{W}_1)
\end{equation}

By similar argument, we have $F_{\mathbf{X}_1}(\mathbf{W}_1) \le F_{\mathbf{X}_2}(\mathbf{W}_2)$. Therefore, $F_{\mathbf{X}_1}(\mathbf{W}_1) = F_{\mathbf{X}_2}(\mathbf{W}_2)$, and their decision rules are the same. Here, the same decision rule means the same discriminant power (trace), but the transformation matrix can differ by a full-rank rotation matrix.
\end{proof}

%-------------------------
\begin{theorem}
For the LDATree model, the node-wise imputation and the root-node imputation will lead to the same decision rule.
\end{theorem}

\begin{proof}
The difference between these two imputation methods is their choice of $C_i$. For numerical variables, the root-node imputation always imputes with the median from the root node, while the node-wise imputation imputes with the median from the current node, leading to a different set of $C_i$.\\

In the LDATree model, we use ULDA for both the node model and split. As long as the current ULDA models are the same for both imputation methods, they will share the same split and child node structure, and the rest will follow. Therefore, we only need to prove that in one particular node, these two imputations yield the same ULDA model while having different choices of $C_i$, which is proved in Lemma \ref{lem:ULDA}.
\end{proof}

Generally speaking, node-wise imputation should yield a different solution compared to root-node imputation. However, we do two things differently: we add the missing value indicators and use the LDA split. Together, these make the two methods equivalent under the LDATree model. Although the two trained models are the same, the testing results could differ if both of the following occur:

\begin{enumerate}
	\item There exists perfect linear dependency in the design matrix.
	\item The new testing data does not comply with the same linear dependency.
\end{enumerate}

During our simulation, this seldom happens. When it does, the difference is negligible. When fitting a FoLDTree model, the difference between the two imputation methods becomes much harder to investigate, as taking a subset of variables changes the column space of the design matrix. The only scenario in which these two imputation methods yield different results is when all of the following occur:
\begin{enumerate}
\item The two imputation methods should impute different values. This means the node-wise median should differ from the overall median.
\item The forward ULDA should select different sets of variables. This means that for at least one variable, one imputation method will treat it as significant while the other will not.
\end{enumerate}

The trickiest part is the second point above. The significance of the variable after imputation depends on the class structure, so neither of these two methods is theoretically admissible. Based on our simulations, the difference between the two methods when fitting FoLDTree is negligible. In our implementation, we set the node-wise imputation as our default missing value solution.

%------------------------------------------------
\subsection{Algorithm Illustration}
\label{subsec:algorithmIllustration}

Here, we use the chessboard pattern to illustrate how our algorithm works. The results are summarized in Figure~\ref{Figure:algorithmIllustration}. The top row shows the original pattern, where no split has been applied, so the sample space remains intact. The pattern is intentionally rotated to challenge traditional axis-orthogonal decision trees and is also difficult for LDA due to its symmetry. To find the first split, we fit an LDA model and use it as the split. Since there are two classes, LDA finds a line that divides the space into two parts, as shown in Figure~\ref{Figure:algorithmIllustration}(c). We now have two subspaces, and we fit LDA models in each subspace, using them as node models. The prediction regions are shown in Figure~\ref{Figure:algorithmIllustration}(d). We then apply the two LDA models as the splitting rule, and the sample space is now divided into four parts, as shown in Figure~\ref{Figure:algorithmIllustration}(e). Since we use LDA for both the node model and splitting rule, it is unsurprising that the shapes in Figure~\ref{Figure:algorithmIllustration}(d) and Figure~\ref{Figure:algorithmIllustration}(e) are identical, differing only in color. Next, we fit four LDA models in the four subspaces, leading to the prediction regions shown in Figure~\ref{Figure:algorithmIllustration}(f). Interestingly, with this new split, the center part is correctly predicted, unlike in Figure~\ref{Figure:algorithmIllustration}(d). As the splitting continues, the sample space is divided into increasingly smaller pieces, and the prediction accuracy improves. Finally, with post-pruning, we obtain the results shown in the last row of Figure~\ref{Figure:algorithmIllustration}, consisting of 157 subspaces in total. We apply the model to a separate test set and achieve a testing accuracy of 0.98. This example demonstrates that integrating LDA into the decision tree framework enables the decision tree to handle more complex patterns effectively.

\begin{figure}[htbp]
    \centering
    \begin{subfigure}[b]{0.34\textwidth}
        \centering
        \includegraphics[width=\textwidth]{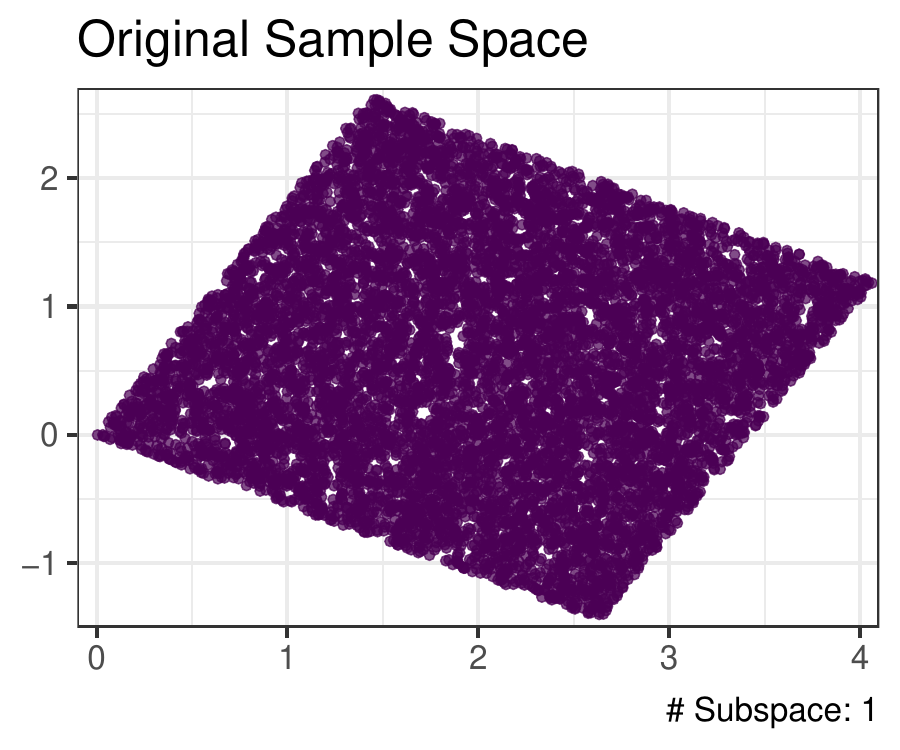}
        \caption{}
    \end{subfigure}%
    \hfill
    \begin{subfigure}[b]{0.34\textwidth}
        \centering
        \includegraphics[width=\textwidth]{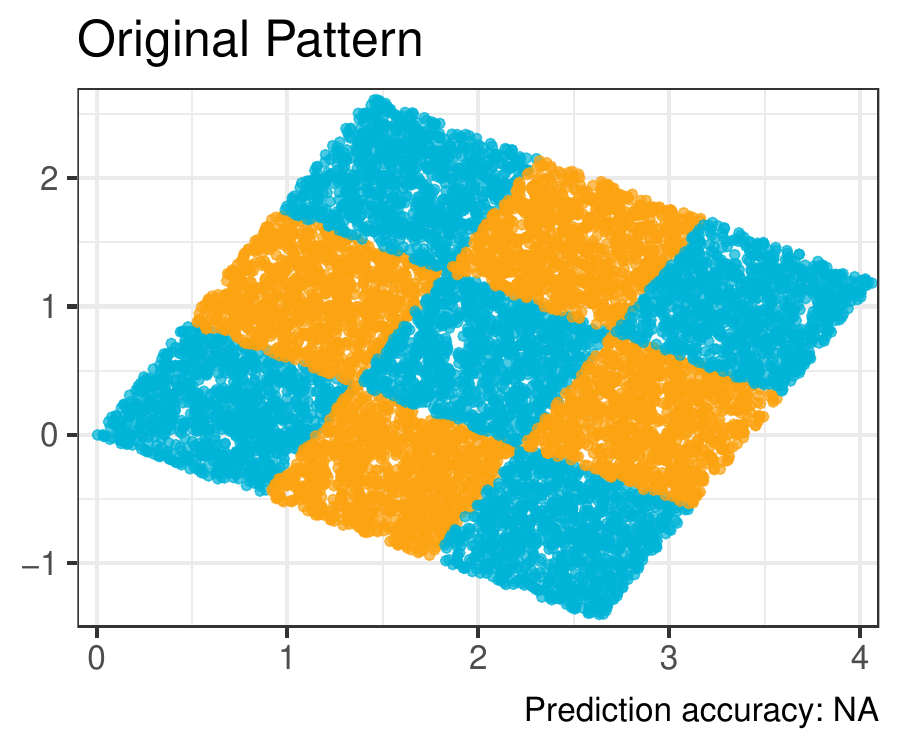}
        \caption{}
    \end{subfigure}
    \begin{subfigure}[b]{0.34\textwidth}
    \centering
    \includegraphics[width=\textwidth]{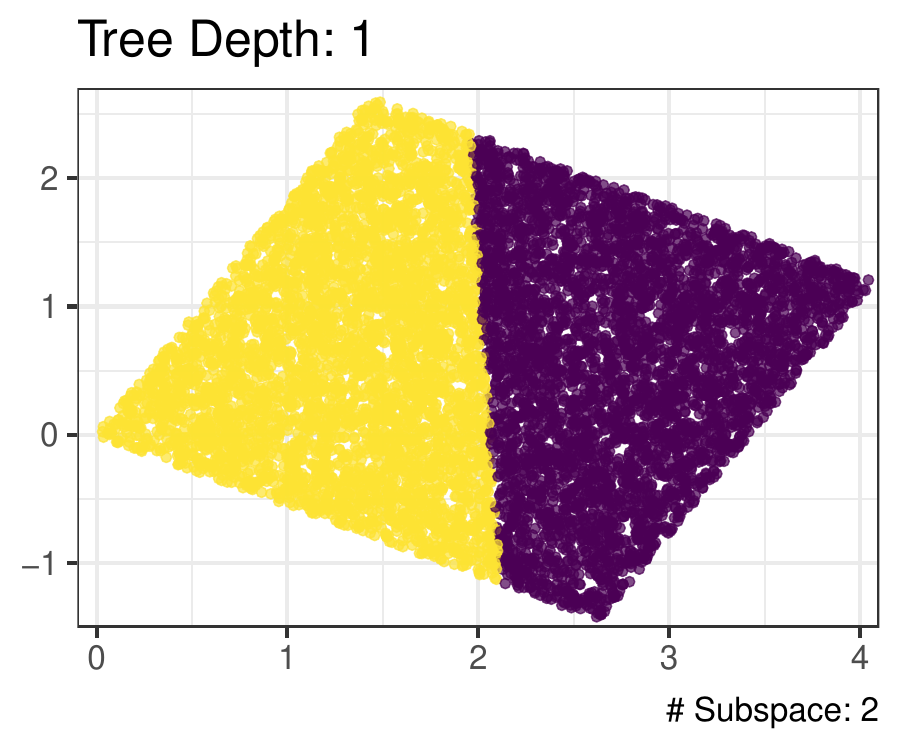}
    \caption{}
    \end{subfigure}%
    \hfill
    \begin{subfigure}[b]{0.34\textwidth}
        \centering
        \includegraphics[width=\textwidth]{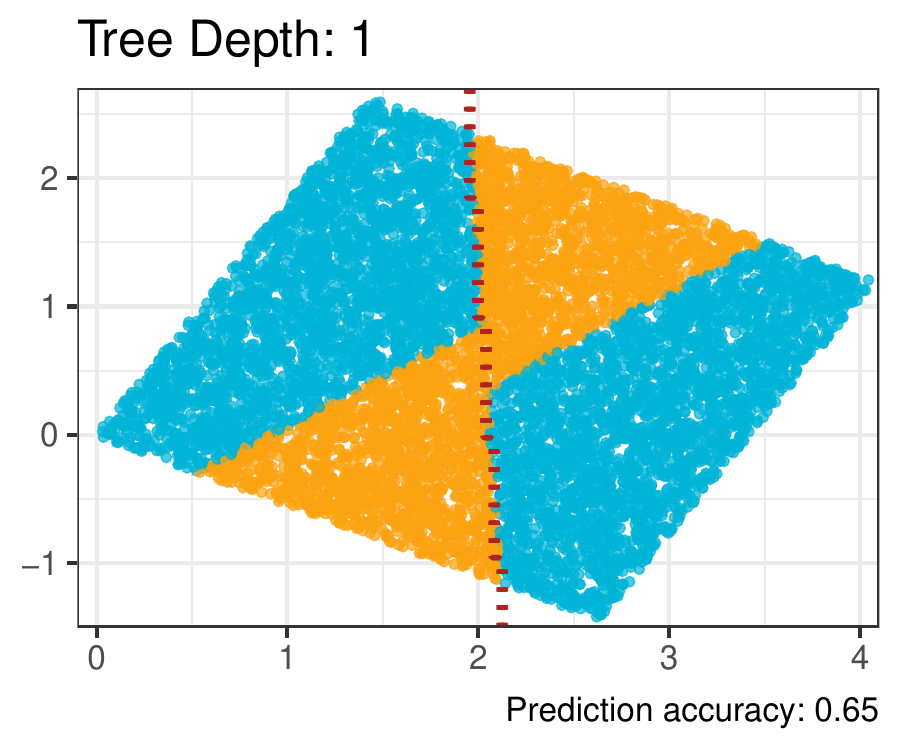}
    \caption{}
    \end{subfigure}
    \begin{subfigure}[b]{0.34\textwidth}
    \centering
    \includegraphics[width=\textwidth]{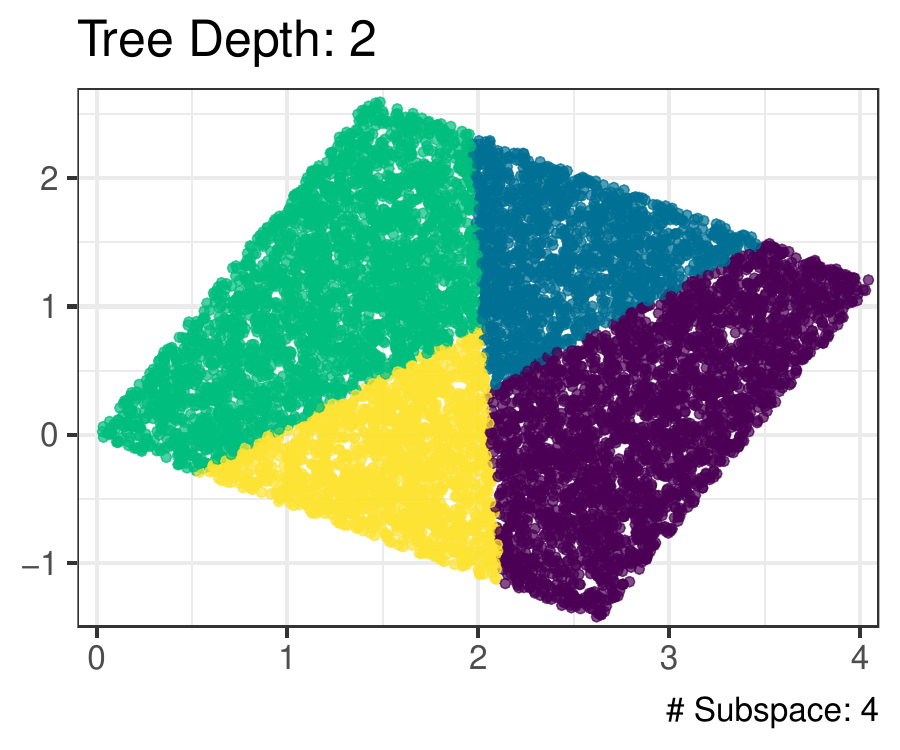}
    \caption{}
    \end{subfigure}%
    \hfill
    \begin{subfigure}[b]{0.34\textwidth}
        \centering
        \includegraphics[width=\textwidth]{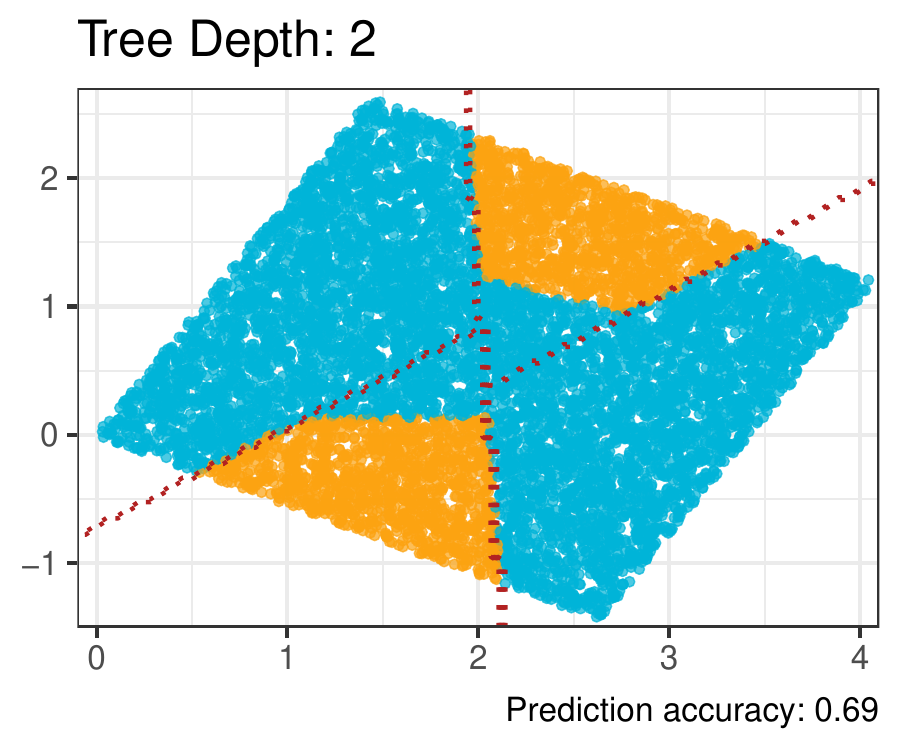}
    \caption{}
    \end{subfigure}
    \begin{subfigure}[b]{0.34\textwidth}
    \centering
    \includegraphics[width=\textwidth]{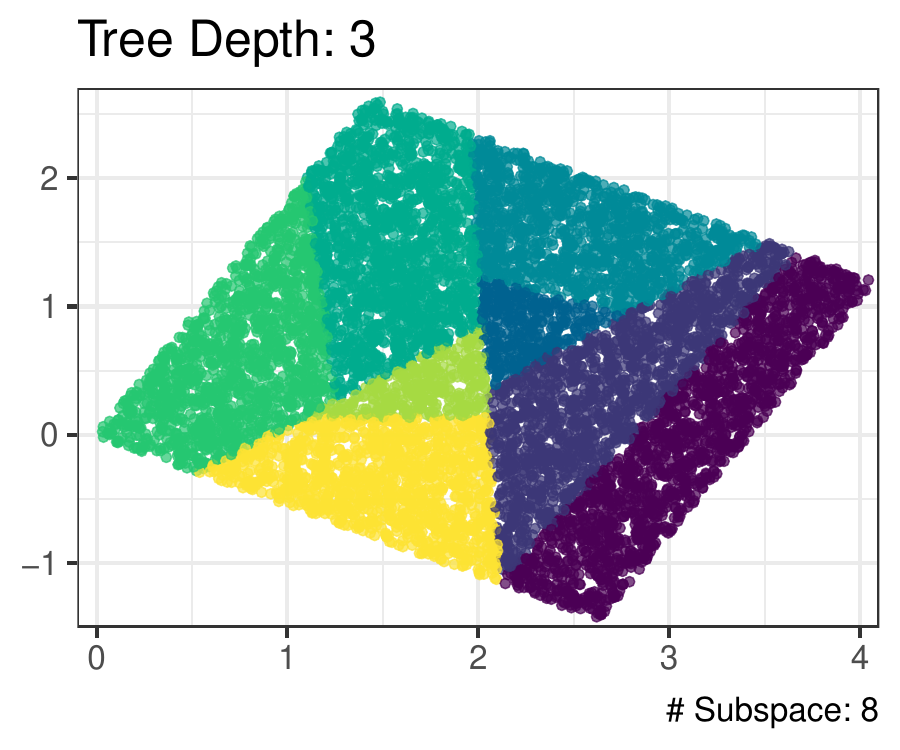}
    \caption{}
    \end{subfigure}%
    \hfill
    \begin{subfigure}[b]{0.34\textwidth}
        \centering
        \includegraphics[width=\textwidth]{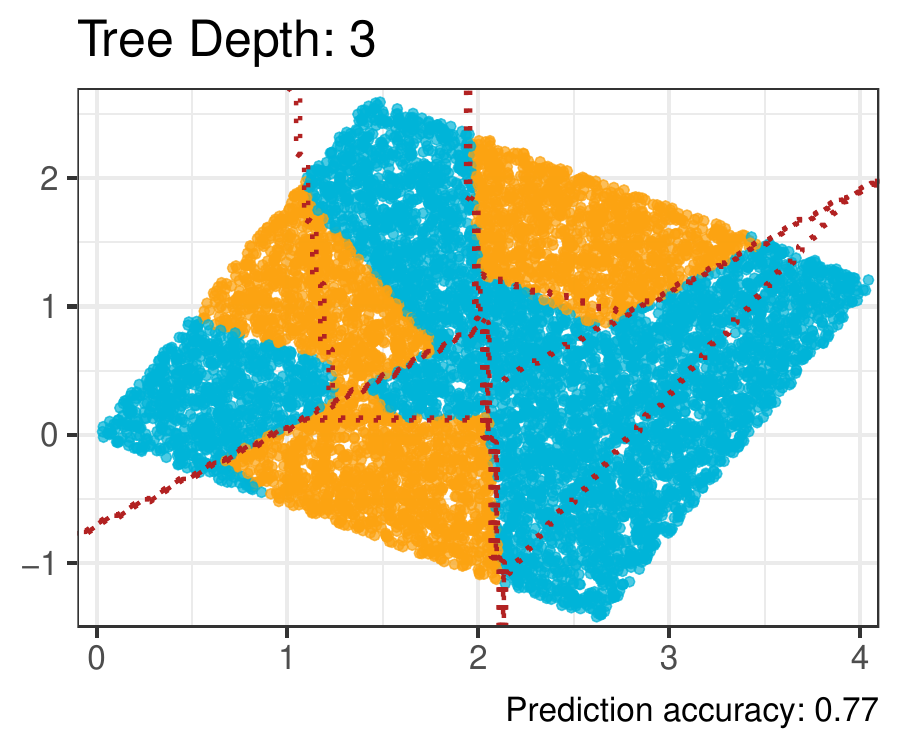}
    \caption{}
    \end{subfigure}
    \begin{subfigure}[b]{0.34\textwidth}
    \centering
    \includegraphics[width=\textwidth]{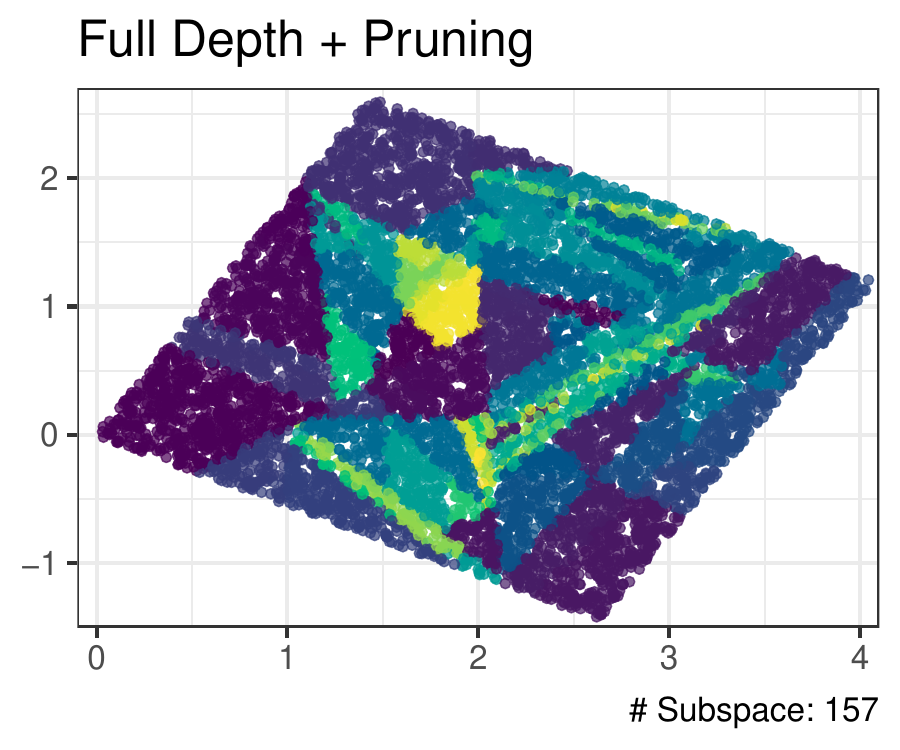}
    \caption{}
    \end{subfigure}%
    \hfill
    \begin{subfigure}[b]{0.34\textwidth}
        \centering
        \includegraphics[width=\textwidth]{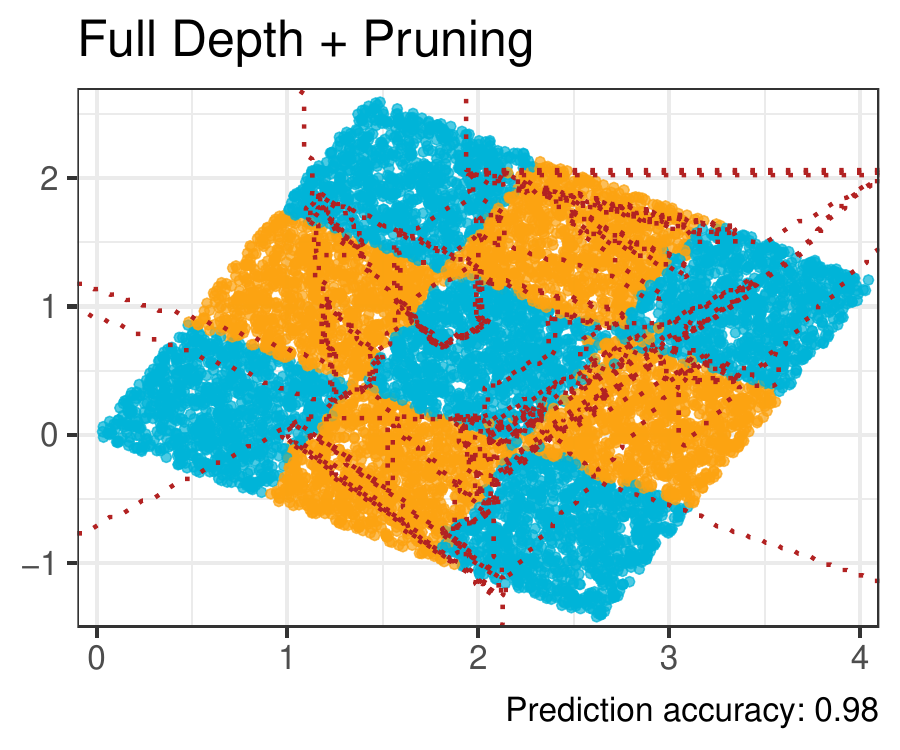}
    \caption{}
    \end{subfigure}
	\caption{Illustration of the LDATree algorithm (Section \ref{subsec:algorithmIllustration}). The right column shows the original and predicted patterns, while the left column shows how the decision tree divides the sample space. The tree is read from top to bottom, with additional splits introduced at each level. The final row presents the post-pruning results}
    \label{Figure:algorithmIllustration}
\end{figure}

%------------------------------------------------
\section{Empirical Analysis}
\label{sec:Empirical}

In this section, we use simulation and real data to demonstrate the performance of our proposed methods and other decision trees, particularly those with oblique splits. Details about the methods used for comparison are provided in Table~\ref{Table:simIntro}. Unless otherwise specified, we use the default parameters for all methods without additional tuning. For LDATree and FoLDTree, we apply the post-pruning stopping rule.

\begin{table}[]
\centering
\begin{tabular}{|l|p{12cm}|}
\hline
Method        & Description                                                                                                                                                                                     \\ \hline
LDATree & Our proposed method using ULDA as the splitting rule and node model. More details are provided in Section~\ref{sec:algorithm}.                        \\ \hline
FoLDTree & Our proposed method using forward ULDA as the splitting rule and node model. More details are provided in Section~\ref{sec:algorithm}.                  \\ \hline
CART & The CART method from the R package \texttt{rpart}. \\ \hline
GUIDE & The GUIDE classification tree \cite{loh2009improving}, with the default splitting option set to linear combinations within the program. \\ \hline
Stree \cite{montanana2021stree} & An oblique decision tree based on SVM. We use the Python package \texttt{STree} for this analysis. \\ \hline
PPtree \cite{lee2013pptree} & The projection pursuit classification tree. We use the R package \texttt{PPtreeViz} for this analysis. \\ \hline
RF & Random forest. We use the R package \texttt{ranger} for this analysis. \\ \hline
\end{tabular}
\caption{Descriptions of the methods tested in Section \ref{sec:Empirical}}
\label{Table:simIntro}
\end{table}

%------------------------------------------------
\subsection{Simulation 1: Robustness to Noise Variables}
\label{subsec:simCase1}

In this simulation, we aim to show that FoLDTree is robust to noise variables and can simultaneously be used as a variable selection tool. This is a useful property, especially if you have many variables.\\

First, we create a chessboard pattern, a 3x3 square with three classes. Each small square contains 2000 uniformly distributed points, resulting in a total of 18,000 points. We split the data 50:50 for training and testing, and the results are summarized in Figure~\ref{Fig:simCase0}. \texttt{RF}, \texttt{LDATree}, and \texttt{FoLDTree} show the highest testing accuracies, followed by \texttt{CART} and \texttt{GUIDE}. \texttt{PPtree} and \texttt{Stree} perform poorly in this scenario.\\

\begin{figure}[htbp]
  	\centering
	\includegraphics[width = 1\textwidth]{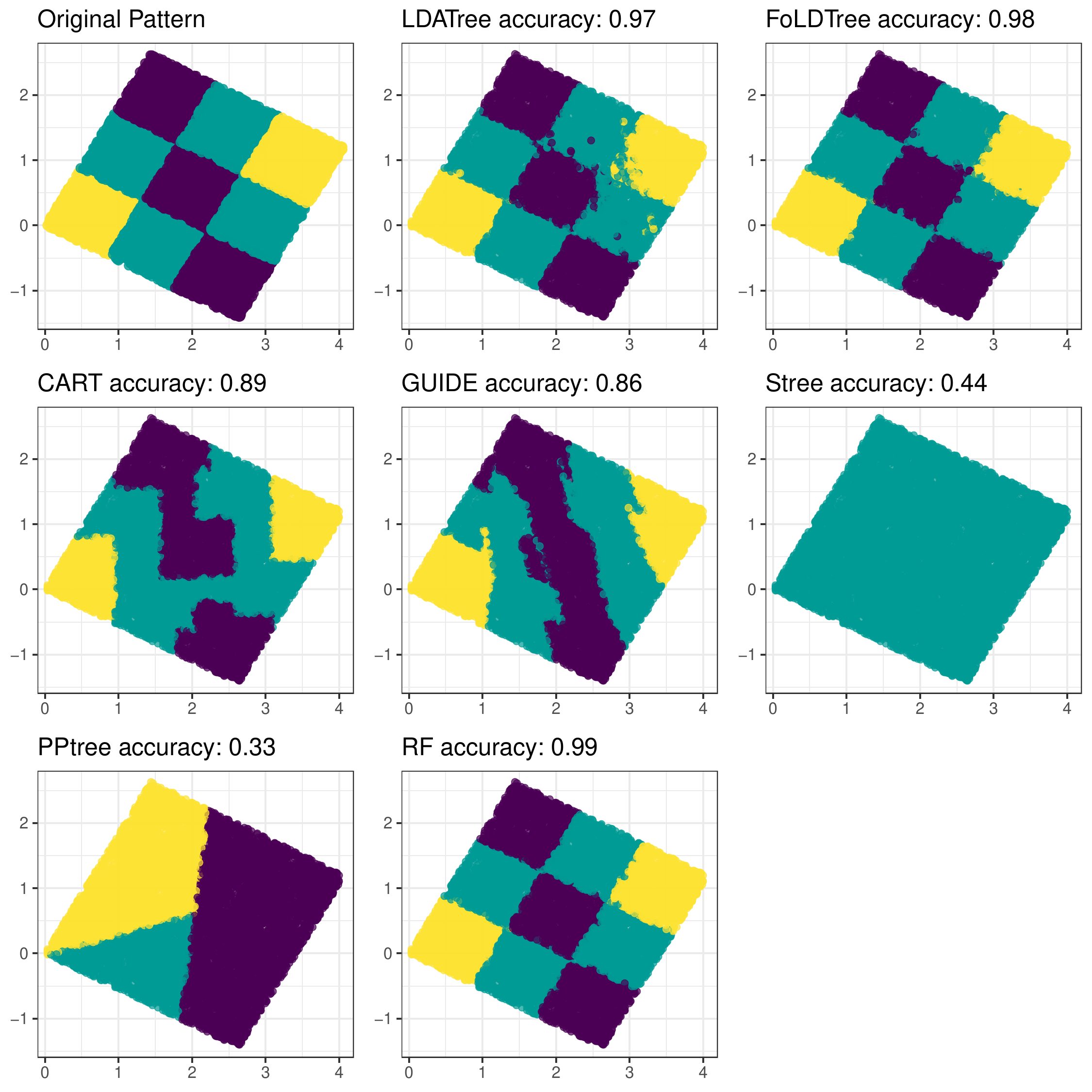}	
	\caption{The simulation results on a 3X3 chessboard pattern (Section \ref{subsec:simCase1}). Except for the original pattern, other plots show the prediction regions.}
	\label{Fig:simCase0}
\end{figure}

Next, we evaluate the change in performance when noise variables are added. In addition to the two informative variables, we add 100 pure noise variables from a standard normal distribution. The results are summarized in Figure~\ref{Fig:simCase1}. As expected, a performance drop occurs due to the noise variables. However, \texttt{FoLDTree}, \texttt{CART}, and \texttt{GUIDE} exhibit minimal performance drop due to their intrinsic variable selection processes. \texttt{LDATree}, which uses all variables to fit LDA, learns from noise variables and thus performs poorly. \texttt{RF} is also significantly affected, as shown by the incorrectly predicted inner purple region in the plot. This is because, during classification, random forest randomly selects $\sqrt{M}$ variables (where $M$ is the number of available variables) at each split, making the selection of informative variables less likely when most are noise. \texttt{Stree} shows an unexpected performance improvement, potentially due to premature stopping in the previous scenario when only two variables are available. These results suggest that \texttt{FoLDTree} is more likely to outperform random forest and other classifiers when numerous noise variables are present.

\begin{figure}[htbp]
  	\centering
	\includegraphics[width = 1\textwidth]{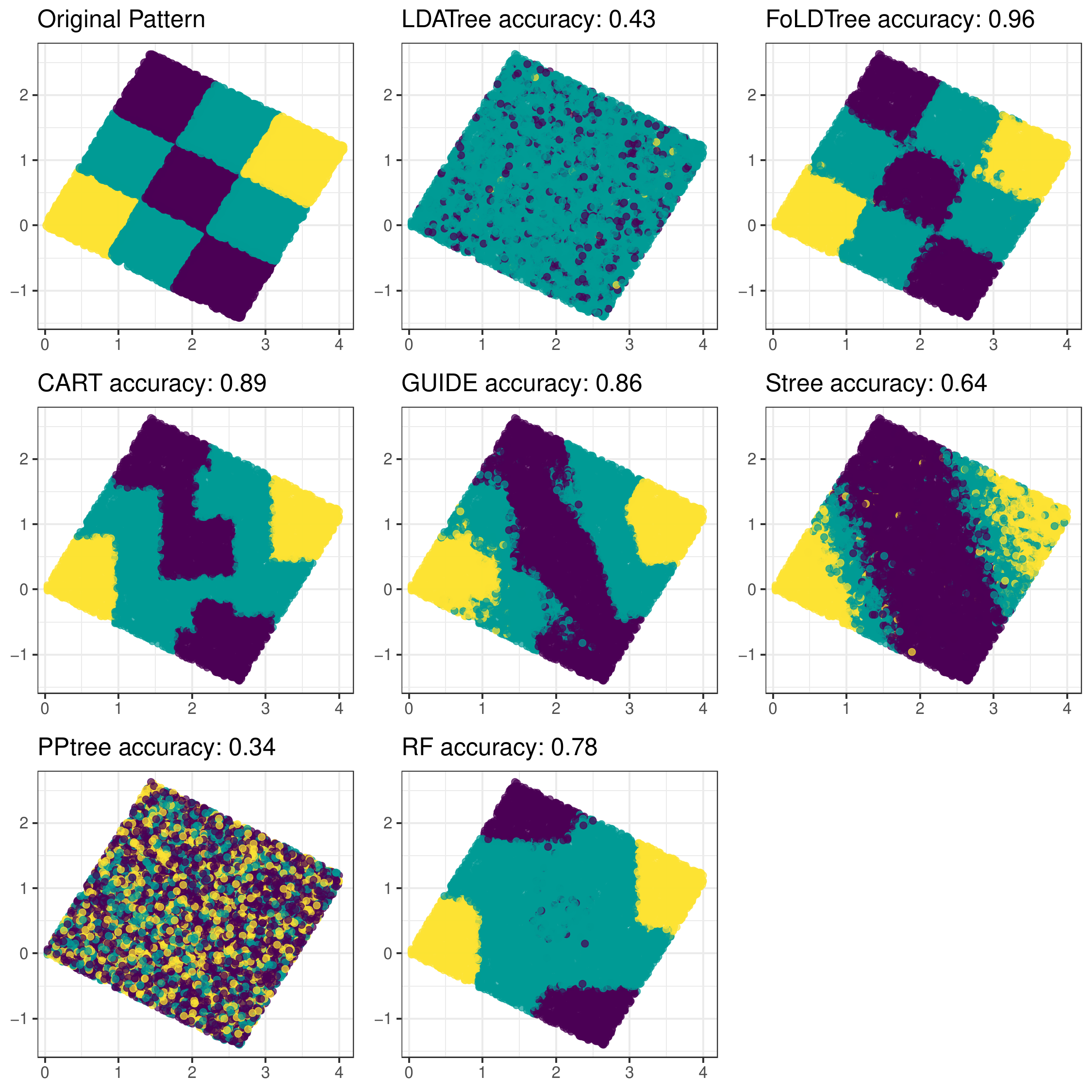}	
	\caption{The simulation results on a 3X3 chessboard pattern with noise variables (Section \ref{subsec:simCase1}). Except for the original pattern, other plots show the prediction regions}
	\label{Fig:simCase1}
\end{figure}

%------------------------------------------------
\subsection{Simulation 2: Look-Ahead Splitting}
\label{subsec:simCase2}

One advantage of the decision tree model is its continuous splitting and fitting process. It is acceptable for the decision tree to fail to find an optimal model at the current node. Recursive splitting shrinks the sample space for each node, often enabling the decision tree to identify better models in the newly generated nodes, as shown in Figure~\ref{Figure:LDAsymmetry}. To enhance performance, we sometimes continue splitting even if the current split temporarily decreases overall performance, with the expectation that subsequent splits will compensate for this drop. Therefore, each decision tree algorithm must determine the extent of this exploration: when a good split is absent, how many additional steps should it look ahead? In this section, we show that both LDATree and FoLDTree can look as many steps ahead as necessary, providing them an advantage over other decision tree methods that use more immediate, greedy strategies.\\

Our simulation is based on the XOR shape in six-dimensional space. The standard XOR problem is in two-dimensional space with two classes. $(0,0)$ and $(1,1)$ return $0$, while $(0,1)$ and $(1,0)$ return $1$. In a 2D plot, this appears as the chessboard pattern shown in Figure~\ref{Fig:simCase0}, but with a 2x2 shape instead of a 3x3. No single cut can significantly decrease the impurity. One step ahead is needed: first, a trivial cut must be made, followed by a more effective cut. For a 3D XOR, two steps ahead are required. As expected, the complexity of the problem increases with dimensionality.\\

To set up the simulation, let the six variables be denoted as $X_1, X_2, \dots, X_6$. Each variable can take a value of either $0$ or $1$, resulting in $2^6 = 64$ centers. We define the response variable $Y$ as $Y = (X_1 + X_2 + X_3 + X_4 + X_5 + X_6) \mod 2$. We sample 100 data points around each center. For a specific center $\mathbf{C} = (c_1, c_2, \dots, c_6)$, points are sampled from $\mathcal{N}(\mathbf{C}, 0.2 \times \mathbf{I}_6)$. This results in a total of $100 \times 64 = 6400$ points. We apply 10-fold cross-validation (CV) to evaluate model performance, with results summarized in Figure~\ref{Fig:simCase2}. Note that confidence intervals are calculated based on the 10 test accuracies from the 10-fold CV.\\

\begin{figure}[htbp]
  	\centering
	\includegraphics[width = 1\textwidth]{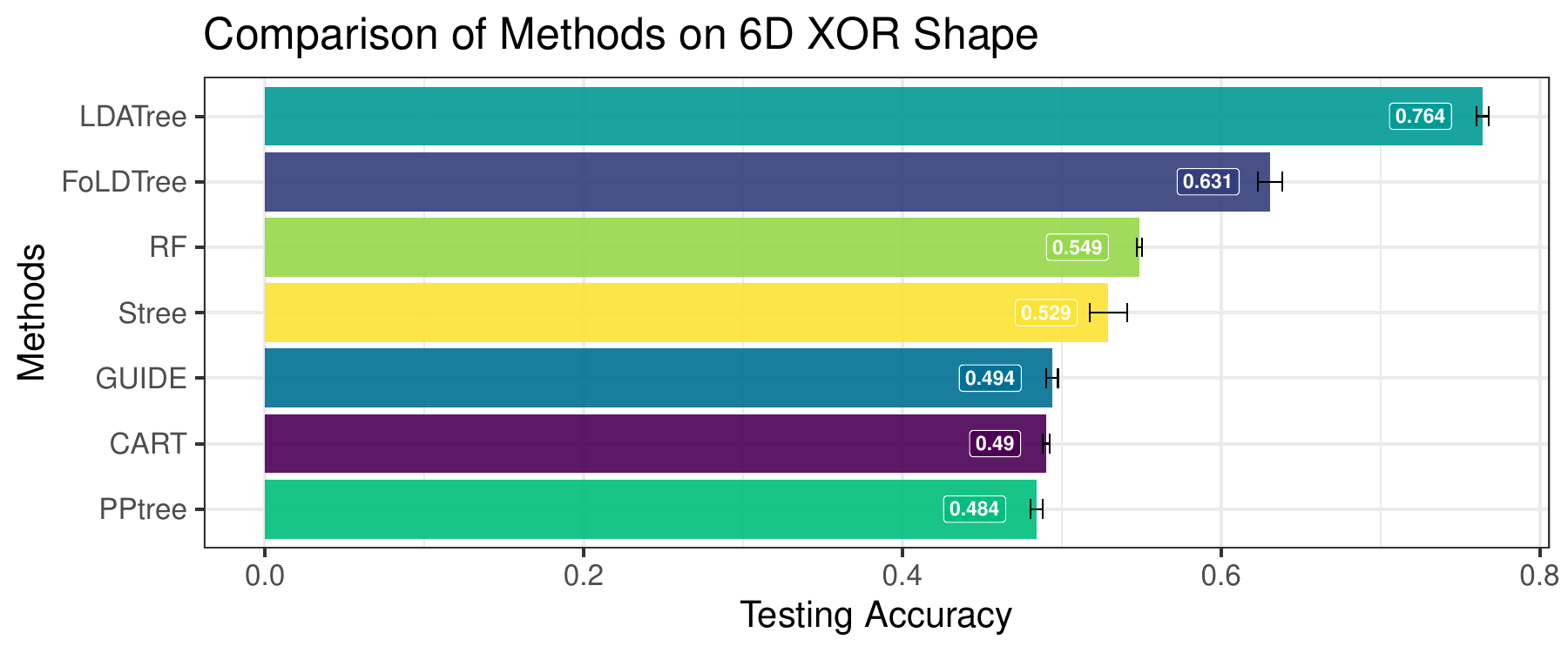}	
	\caption{The testing results (6D XOR) in Section \ref{subsec:simCase2}. Methods are ordered by their accuracies, with confidence intervals for accuracies shown in 2SD error bars.}
	\label{Fig:simCase2}
\end{figure}

\texttt{LDATree} has the highest testing accuracy, followed by \texttt{FoLDTree}. This difference arises because the 6D XOR pattern involves a significant 6th-order interaction but non-significant main effects and 2nd-order to 5th-order interactions. Therefore, the forward selection in \texttt{FoLDTree} may fail to select effective variables in certain nodes, while \texttt{LDATree} uses all the variables in all nodes. \texttt{GUIDE}, \texttt{CART}, and \texttt{PPtree} have the lowest testing accuracies, performing similarly to the plurality rule ($0.5$). \texttt{RF} and \texttt{Stree} are in the middle. Thus, we conclude that for problems involving high-order interactions and requiring proactive searching, LDATree and FoLDTree perform better compared to the random forest and other single-tree methods.

%------------------------------------------------
\subsection{Real Data Analysis}
\label{subsec:realData}

Here, we evaluate the models on nine real datasets from the UC Irvine Machine Learning Repository \cite{UCI_ML_Repository}. We conducted the experiment 20 times for each dataset, randomly splitting the data into training and testing sets with a 70:30 ratio each time. For methods that cannot handle missing values, we impute the missing values using the simple imputation method mentioned in Section~\ref{subsec:NA}. We use testing accuracy as our main metric, and the results are summarized in Table~\ref{Table:realMethods}.\\

\begin{sidewaystable}
\centering
\small
\begin{tabular}{|l|l|l|l|l|l|l|l|l|l|l|}
\hline
Dataset             & \#S & \#F & \#L & CART          & FoLDTree      & GUIDE         & LDATree                & PPtree                & RF                     & Stree                  \\ \hline
SUPPORT2            & 9105      & 64         & 2        & 0.886 (0.002) & 0.935 (0.002) & 0.935 (0.004) & 0.926 (0.004)          & 0.858 (0.002)         & \textbf{0.939 (0.001)} & 0.588 (0.003)          \\ \hline
Polish Bankruptcy   & 4182      & 64         & 2        & 0.716 (0.006) & 0.78 (0.006)  & 0.791 (0.009) & 0.821 (0.006)          & 0.594 (0.01)          & \textbf{0.837 (0.004)} & 0.689 (0.006)          \\ \hline
Mice Protein        & 1080      & 77         & 8        & 0.726 (0.009) & 0.987 (0.004) & 0.797 (0.023) & 0.986 (0.004)          & 0.988 (0.003)         & \textbf{0.99 (0.003)}  & 0.981 (0.005)          \\ \hline
Haberman Survival   & 306       & 3          & 2        & 0.716 (0.017) & 0.726 (0.018) & 0.717 (0.015) & 0.728 (0.018)          & \textbf{0.74 (0.018)} & 0.728 (0.015)          & 0.736 (0.022)          \\ \hline
Balance Scale       & 625       & 4          & 3        & 0.79 (0.01)   & 0.881 (0.009) & 0.873 (0.012) & \textbf{0.907 (0.009)} & 0.819 (0.014)         & 0.853 (0.012)          & 0.901 (0.008)          \\ \hline
South German Credit & 1000      & 53         & 4        & 0.492 (0.01)  & 0.491 (0.014) & 0.5 (0.009)   & 0.478 (0.012)          & 0.402 (0.012)         & \textbf{0.519 (0.011)} & 0.431 (0.012)          \\ \hline
Dry Bean            & 13611     & 16         & 7        & 0.872 (0.002) & 0.921 (0.001) & 0.905 (0.003) & 0.921 (0.002)          & 0.898 (0.002)         & 0.924 (0.002)          & \textbf{0.924 (0.001)} \\ \hline
NATICUSdroid        & 29332     & 86         & 2        & 0.928 (0.001) & 0.964 (0.001) & 0.942 (0.002) & 0.962 (0.001)          & 0.948 (0.001)         & \textbf{0.969 (0.001)} & 0.958 (0.001)          \\ \hline
Breast Cancer WI    & 569       & 30         & 2        & 0.928 (0.007) & 0.951 (0.008) & 0.947 (0.008) & 0.946 (0.007)          & 0.95 (0.007)          & 0.958 (0.007)          & \textbf{0.968 (0.006)} \\ \hline
Average Accuracy    &           &            &          & 0.784 (0.003) & 0.848 (0.003) & 0.823 (0.004) & 0.853 (0.003)          & 0.8 (0.003)           & \textbf{0.857 (0.003)} & 0.797 (0.003)          \\ \hline
Average Rank        &           &            &          & 6.11 (0.32)   & 3.44 (0.2)    & 4.33 (0.37)   & 3.56 (0.34)            & 4.89 (0.48)           & \textbf{1.89 (0.3)}    & 3.78 (0.52)            \\ \hline
\end{tabular}
\caption{Testing accuracy results for all the algorithms and datasets (Section \ref{subsec:realData}). Two standard deviations are noted in parentheses. The second to fourth columns represent the number of samples, the number of features, and the number of levels of the response variable. The last two rows show the average testing accuracies and average ranks calculated from the nine datasets above.}
\label{Table:realMethods}
\end{sidewaystable}

The best method with respect to testing accuracy is \texttt{RF}, followed by \texttt{LDATree} and \texttt{FoLDTree}. The two proposed methods demonstrate consistently reliable performance across all datasets. When calculating the confidence interval for the average testing accuracy, \texttt{LDATree}'s interval overlaps with \texttt{RF}'s, and \texttt{FoLDTree}'s overlaps with \texttt{LDATree}'s. In contrast, the other four methods perform poorly on at least one dataset. Looking at the rank of these methods, \texttt{RF} ranks highest, followed by \texttt{FoLDTree} and \texttt{LDATree}. These results demonstrate the superiority of the proposed methods over \texttt{Stree} and \texttt{PPtree}, with performance nearly matching that of \texttt{RF}, which is acceptable given that the proposed method consists of only one tree. Notably, the features in these datasets are likely to be manually selected, which may explain why \texttt{FoLDTree} does not significantly outperform \texttt{LDATree}, as there are few noise variables.

%------------------------------------------------
\section{Conclusion}
\label{sec:Conclusion}

In this paper, we present two new decision tree frameworks that use ULDA and forward ULDA to find splits. These frameworks can effectively generate oblique splits, handle missing values, perform feature selection, and output predicted class labels and class probabilities. Based on our real data analysis, they generally outperform traditional decision trees with orthogonal splits, such as CART, and other oblique trees like PPtree and Stree, while achieving performance comparable to that of the random forest. We also identify two specific use cases for the proposed method: when the dataset contains many noise variables, and when significant high-order interactions are present alongside non-significant low-order interactions. In these scenarios, the proposed method outperforms other methods, including the random forest.\\

One potential future direction relates to how splits are generated. Currently, when the LDA split is not effective, it will serve as a random split. In simulations, we observe that it sometimes takes three or more splits for LDATree to approximate complex patterns, such as nonlinear decision boundaries or symmetrical patterns. This is not very efficient. In the future, we could allow it to fall back to a univariate split when no effective ULDA split is found. Additionally, we could incorporate insights from classifiers sensitive to decision boundaries, such as SVM.\\

Overall, this paper presents a novel approach to integrating LDA into the decision tree framework, addressing several bottlenecks from previous attempts in the literature. We encourage researchers to further explore this direction, such as by developing an ensemble version of LDATree. The related R package, \texttt{LDATree}, is available on CRAN.

\clearpage % For section ends, add a new page and enforce all floatings to appear

\bibliographystyle{plainnat}
\bibliography{papers}

\end{document}